\documentclass{article}
\usepackage{geometry}

\usepackage[utf8]{inputenc}
\usepackage{graphicx} 
\usepackage{float} 
\usepackage{subfigure} 
\usepackage{amsmath}
\usepackage{biblatex}

\usepackage{hyperref}
\usepackage{url}

\usepackage{multirow}
\usepackage{booktabs}
\usepackage{amsthm}
\usepackage{multirow}

\newtheorem{theorem}{Theorem}
\newtheorem{corollary}{Corollary}
\newtheorem{proposition}{proposition}
\newtheorem{remark}{remark}

\title{FedMAE: Federated Self-Supervised Learning with One-Block Masked Auto-Encoder}
\author{Nan Yang, Xuanyu Chen, Charles Z. Liu, Dong Yuan, Wei Bao and Lizhen Cui \\
\textit{Faculty of Engineering, The University of Sydney, Australia}\\
\{n.yang, xuanyu.chen, zhenzhong.liu, dong.yuan, wei.bao\}@sydney.edu.au \\
\textit{School of Software Engineering, Shandong University, China}\\
clz@sdu.edu.cn
}

\date{}

\begin{document}
\maketitle

\begin{abstract}
Latest federated learning (FL) methods started to focus on how to use unlabeled data in clients for training due to users' privacy concerns, high labeling costs, or lack of expertise. However, current Federated Semi-Supervised/Self-Supervised Learning (FSSL) approaches fail to learn large-scale images because of the limited computing resources of local clients. In this paper, we introduce a new framework FedMAE, which stands for Federated Masked AutoEncoder, to address the problem of how to utilize unlabeled large-scale images for FL. Specifically, FedMAE can pre-train one-block Masked AutoEncoder (MAE) using large images in lightweight client devices, and then cascades multiple pre-trained one-block MAEs in the server to build a multi-block ViT backbone for downstream tasks. Theoretical analysis and experimental results on image reconstruction and classification show that our FedMAE achieves superior performance compared to the state-of-the-art FSSL methods.

\end{abstract}

\section{Introduction}
Self-supervised learning (SSL) has attracted extensive research attention for learning representations without costly data labeling. The common practice in computer vision is to design proxy tasks to allow visual representation learning from unlabeled images \cite{doersch2015unsupervised,noroozi2016unsupervised,zhang2016colorful,gidaris2018unsupervised}. The current state-of-the-art SSL methods integrate contrastive learning with Siamese networks to increase the similarity between two augmented versions of images \cite{wu2018unsupervised,chen2020simple,he2020momentum,grill2020bootstrap,chen2021exploring}, or use autoencoder to reconstruct the input images \cite{dosovitskiy2020image,bao2021beit,he2022masked}, while all these solutions assume training images are centrally available in cloud servers.
In the real world, however, decentralized image data are exploding in growth, and data collected by different parties may not be centralized due to data privacy regulations.

Existing approaches cannot well use decentralized unlabeled data to learn a generic representation while preserving data privacy. Federated Learning (FL) \cite{chen2019communication,mcmahan2017communication}is an emerging distributed training technique in which several clients collectively train a global model through coordinated communication.
However, traditional federated learning \cite{kairouz2021advances} relies on data labels, while this assumption may be excessive as labeling data is expensive, time-consuming, and indispensable to the participation of domain experts. Decentralized data is normally unlabeled, non-independent and identically distributed (non-IID), which are critical challenges for learning from multiple clients \cite{li2020federated,zhao2018federated}. Motivated by these practical scenarios, many SSL methods are used in federated learning, such as FedU \cite{zhuang2021collaborative}, FedEMA \cite{zhuang2021divergence} and Orchestra \cite{DBLP:conf/icml/LubanaTKDM22}. Due to the limited computation resource of client devices, these approaches use small models, e.g., ResNet18, to train small images. The state-of-the-art methods have not revealed how lightweight clients learn unlabeled large-scale images effectively.

Recent advancement of Transformers has great achievements in computer vision \cite{dosovitskiy2020image,he2022masked}. However, the biggest problem with applying Transformers in federated learning is a large weight, which is difficult to deploy on small computing devices for training large images. In this paper, we propose a simple, effective, and scalable framework, FedMAE, which stands for Federated Masked AutoEncoder. Our FedMAE only deploys the MAE with a one-block encoder and a one-block decoder as our pre-training model on the local clients and masks the random patches in the input image and then reconstructs the missing patches in the pixel space. Although the one-block model learns limited information and performs poorly in downstream tasks, we found that cascading many pre-trained one-block MAEs together (without weights aggregation in traditional FL) can achieve stunning performance in downstream tasks. With this design, FedMAE can achieve a win-win scenario, i.e., it improves model accuracy by training large images on lightweight client devices and allowing more distributed or mobile clients to participate in training asynchronously. To get a qualitative sense of our FedMAE performance, see Figure \ref{fig1}, which compares the reconstruction effect of the un-pretrained five-block MAE and the five-block MAE pre-trained by FedMAE. 

In summary, the main contributions of this work are summarized as follows:
\begin{itemize}
\item We propose FedMAE, a framework that realizes asynchronous training of large-scale unlabeled images in FL. To the best of our knowledge, we are the first to apply Transformer to FL, which shows that cascading multiple pre-trained one-block MAEs can form a high-performance pre-trained ViT model for downstream tasks.

\item 
We model an approximate mapping for FedMAE based on the equivalence between data corruption and feature interference and give an approximate analytical solution to optimal reconstruction, which depends only on the data, revealing the data impact of the optimal reconstruction in federated autoencoding learning independent of the backbone.

\item We pre-train FedMAE with large-scale datasets in FL and conduct extensive fine-tuning experiments on downstream tasks, such as image reconstruction and image classification. The results show FedMAE has huge improvement over state-of-the-art baselines.


\begin{figure}
  \centering
  \includegraphics[width=1.0\textwidth]{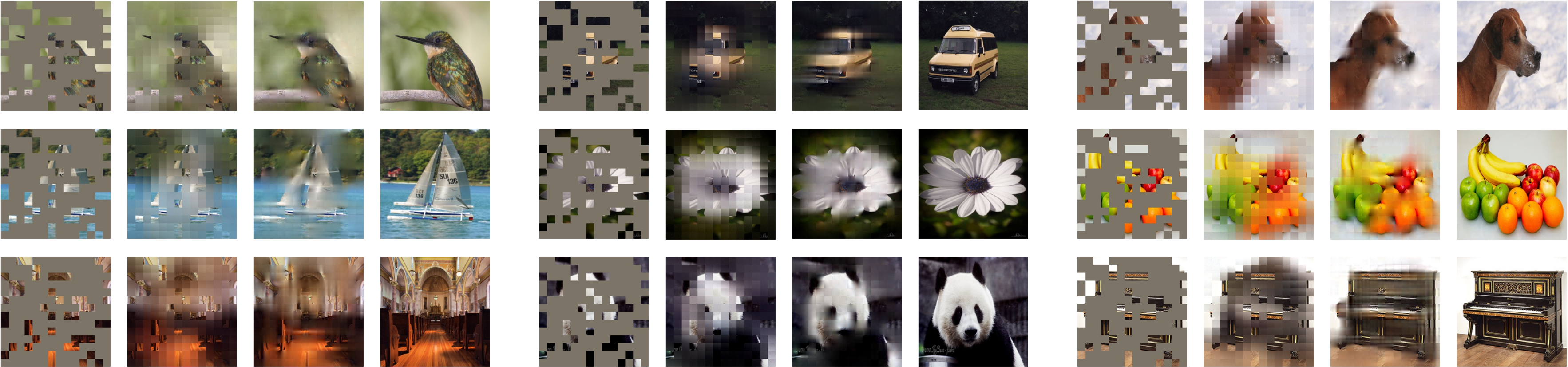}
  \caption{Comparison of reconstructions of ImageNet validation images under 75\% masking ratio. From left to right of each quaternion, we show the masked image, the reconstruction of the un-pretrained five-block MAE, the reconstruction of the five-block MAE pre-trained by FedMAE, and the ground-truth image.}
  \label{fig1}
\end{figure}

\end{itemize}

\section{Related Work}
\subsection{Self-Supervised Learning and Transformers}
Self-Supervised Learning is a recent paradigm in unsupervised representation learning, especially discriminative SSL, which exploits the invariance of different views after data augmentation, making the representations of positive pairs similar while pushing negative pairs away. In order to obtain sufficient informative negative examples in contrastive learning, the methods usually rely on large memory banks such as Moco \cite{he2020momentum}, or large batch size such as SimCLR\cite{chen2020simple}. BYOL \cite{grill2020bootstrap} and SimSiam \cite{chen2021exploring} further minimize the need for negative samples by using a variety of strategies to prevent representation collapse. 

Pre-training vision Transformers has attracted great attention recently due to the data-hungry issue. ALBERT utilizes two parameter reduction techniques that overcome the primary challenges of scaling pre-trained models, resulting in a light and efficient language representation model \cite{lan2019albert}. ViT \cite{dosovitskiy2020image} predicts the 3-bit mean color of the masked patches, which is the most straightforward translation of BERT \cite{devlin2018bert} from NLP to CV. BEIT \cite{bao2021beit} uses image patches as input and visual tokens are obtained by discrete VAE instead of clustering. The current most popular way of pretraining Transformer is MAE \cite{he2022masked} and MaskFeat \cite{wei2022masked}, which utilize the idea of self-supervised learning to predict features of the masked area.

\subsection{Federated Semi-Supervised/Self-Supervised Learning}

Semi-supervised learning and self-supervised learning are two methods in federated learning to utilize unlabelled data in client devices for training.
Federated semi-supervised learning has to rely on some labeled data either in the server or in the devices (or both) to guide unlabeled data in training the global model. RSCFed \cite{liang2022rscfed} relies on clients with labeled data to assist in learning unlabeled data in other clients. FedMatch \cite{jeong2020federated} introduced the inter-client consistency that aims to maximize the agreement across models trained at different clients. 

Federated self-supervised learning normally does not rely on any labeled data but aims at pre-training a high-performance model for downstream tasks. FedU \cite{zhuang2021collaborative} is based on the self-supervised method BYOL \cite{grill2020bootstrap}, which aims for representation learning. FedEMA \cite{zhuang2021divergence} is an upgraded version of FedU, which adaptively updates online networks of clients with EMA of the global model. Orchestra \cite{DBLP:conf/icml/LubanaTKDM22} relies on high representational similarity for related samples and low similarity across different samples.

\begin{figure}
  \centering
  \includegraphics[width=1.0\textwidth]{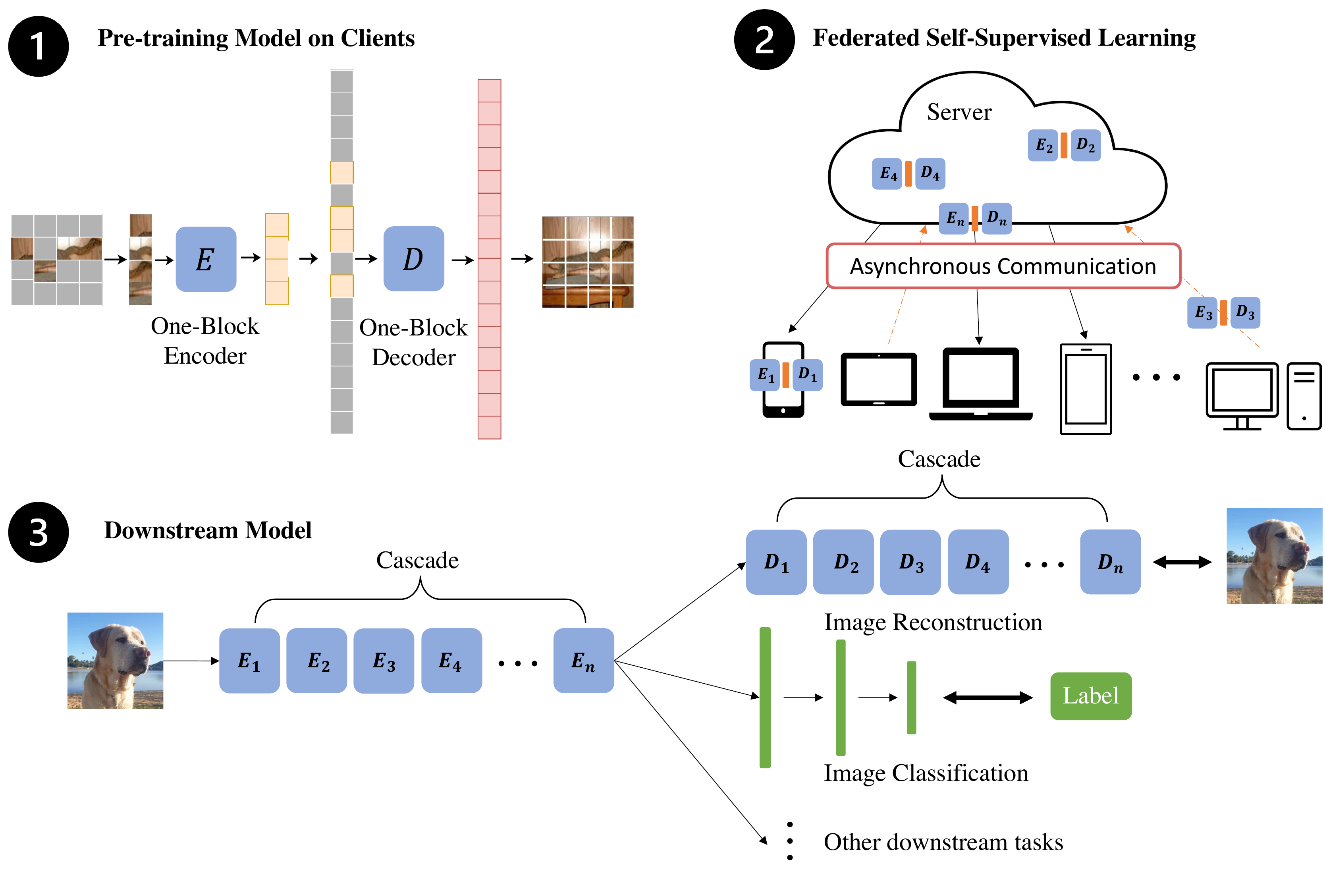}
  \caption{Overview of FedMAE Architecture. \textcircled{1} We use a tiny MAE with a one-block encoder and a one-block decoder as our pre-training model on clients.
\textcircled{2} Pre-training models in FL can be trained on different clients asynchronously without average aggregation.
\textcircled{3} After pre-training in FL, pre-trained models are cascaded in the server to build a multi-block ViT for downstream tasks.}
  \label{fig2}
\end{figure}

\section{Approach}
A masked autoencoder is a simple autoencoding method that reconstructs the original signal given its partial observation. To enable more local clients to participate in training, FedMAE's pre-training model contains only one block of encoder and decoder, allowing many devices with insufficient computing power, e.g., mobile phones and Internet of Things devices, to participate in training on large images. Like all autoencoders, our one-block encoder maps the observed signal to a latent representation, and the one-block decoder reconstructs the original signal from the latent representation when pre-training. For the downstream model of FedMAE, we adopt a cascade design that allows multiple pre-trained one-block MAE connected in sequence to build a ViT model with multi-blocks for downstream training. Figure \ref{fig2} illustrates the idea, introduced next.

\subsection{Pre-trained Model}
In FedMAE, we build a tiny MAE with a one-block encoder and a one-block decoder as the pre-trained model for client devices. Our encoder is a ViT but applied only on visible, unmasked patches. Similar to a traditional ViT, our encoder embeds patches using a linear projection with additional positional embeddings and then processes the resultant set using a single Transformer block. Just as in a standard MAE, our encoder only performs on a small subset (e.g., 25\%) of the whole set. Masked patches are removed, and no mask tokens are used for our encoder. The whole set is handled by a lightweight one-block decoder. The pre-trained model only has one Transformer block, which is used to predict masked patches for reconstructing original images.

\subsection{Downstream Model}

Without a doubt, the pre-trained model has a certain ability to filter interference that masks features because of the loss of explicit information. Therefore, when the pre-trained model is transferred to the downstream ViT classifier, its performance is significantly improved compared to the un-pretrained ViT, which can only obtain classification knowledge through label training and cannot mine implicit features. Moreover, the input and output of each block from the Transformer are in the same dimension, and each block in the multi-block ViT is learning how to extract feature representations. They have the same or similar learning capabilities, although there are some differences in extracting features. Due to the similarity of learning purposes for each block, FedMAE tries to cascade multiple one-block pre-trained MAEs, which will achieve higher-order implicit feature representations based on multi-block mapping to help boost and speed up the downstream task training.

\subsection{Training Process}
In Figure \ref{fig2}, \textcircled{1} shows the pre-training model of FedMAE on clients, which only contains a one-block encoder and a one-block decoder. We use unmasked patches (25\% of the whole image) for encoding, which means the model is tiny that is easy to deploy in many lightweight clients and train large images.
\textcircled{2} shows the communication between clients and the server. Due to the special one-block structure of the pre-training model and the very small learning rate of Transformer, the average operation of multiple model weights is not necessary (or effective) to boost the accuracy. Therefore, FedMAE realises pure asynchronous training, i.e., multiple pre-training models can train on different clients asynchronously without relying on each other. \textcircled{3} shows cascading multiple pre-trained one-block MAEs to form a multi-block model to help train the downstream task, e.g., both the encoder and decoder are needed to perform the image reconstruction task, and only the encoder is used to produce image representations for recognition.

\section{Analysis}

Based on the FedMAE as a system prototype, we analyze the data modeling and system mapping of federated learning with autoencoding. We analyze the relationship between the absence of explicit information in the data and feature perturbations and establish the equivalence between spatial variation in the data and feature denoising. An analytic model for federated autoencoders is given based on approximate linear mapping with a corresponding analytic solution as the linear approximation of optimal reconstruction.
This analytic solution depends on the data only, revealing the data impact of the optimal reconstruction independent of the backbone in federated learning.

\subsection{Problem Formulation}

\subsubsection{Data Corruption and Feature interference}
Modeling the pre-training dataset as a source dataset $D_S=\{x_{S1},x_{S2},\ldots,x_{Sn}\}$ and downstream dataset as a target dataset $D_T=\{x_{T1},x_{T2},\ldots,x_{Tn}\}$ with lables $Y_T=\{y_{T1},y_{T2},\ldots,y_{Tn}\}$, we design an encoder $h(x)=z$ and a decoder $g(z)=x$.
The reconstruction by the pair of $h(x),g(z)$ can be represented as $\hat{x}=g(h(x))$.
We use a loss function $l(x, g(h(x)))$ to measure the error between the original and reconstructed data, while the parameters of the autoencoder are learned through $D_S$ to minimize the reconstruction error.

Since MAE uses patch sampling, the $x$ input loses much explicit information, degraded as $\tilde{x}$, which leads to a corrupted input-to-feature explicit feature knowledge representation of $g(\cdot),h(\cdot)$. At this point, the training task becomes into reconstructing the original input $x$ from its corrupted degradation $\tilde{x}$ by minimizing $l(x, g(h(\tilde{x}))$. While on the features, the loss of explicit information corresponds to the loss of a portion of the features, leading to the corrupted degradation of the features as $\tilde{z}$. Since we use uniform dimensional encoding in the features, $\tilde{z}$ and $z$ are still uniform dimensions, so the above corruption in the features can be modeled as an interference as $\Delta z = z - \tilde{z}$.

It can be seen that the loss of explicit information in data $x$ is equivalent to the noise interference of feature $z$. We analyze this feature interference due to data loss and manage to design a mechanism to reduce this interference by hidden feature mining and train the system to learn for discovering hidden features.

\subsubsection{Federated Learning with Variation Based Augmentation}
We adopt FL in the pre-training system. There are $K$ clients and 1 server, where each client has $n$ samples as a dataset $\textbf{X}_k=\{x_{ki}\}_n$, where $k$ refers to the index of the $k$th client. If we aggregate all the samples and format them into a package, the corresponding aggregation can be formulated as $X_k=[x_{k1},\ldots,x_{kn}]$.
With MAE, the system training uses random corruption, i.e., the location of the deducted patches is different each time.
Therefore, the original data on $k$th client then has multiple variant versions of $\tilde{\textbf{X}}_{k,j} =\{x_{ki,j}\}_{n\times m}$ with the corresponding aggregation as $\tilde{X}_{k,j}=[\tilde{x}_{k1,j},\ldots,\tilde{x}_{kn,j}]$,
where $\tilde{x}_{ki,j}$ refers to the variation result of the original input $x_{ki}$ with the $j$th masked corruption.

When $m$ times of masked corruption is applied, then $\tilde{X}_{k,j}, k\leq K, j\leq m$ can provide $K\times n \times m $ samples of variants. It is equivalent to subjecting $K\times n$ original samples to $m$ corruption variants.
It can be seen that the original data are augmented with features interfering with corruption processing to form diverse variations. This strategy can improve the sample diversity as an augmentation on the one hand and reduce the variance on the other hand so that the system can learn the impact of feature interference on the data caused by multiple explicit information loss.

\subsection{Federated Learning with Autoencoder}
For a single masked autoencoder, we model the corrupted inputs as $\tilde{D_S}=\{\tilde{x}_{S1},\tilde{x}_{S2},\ldots,\tilde{x}_{Sn}\}$, where each $\tilde{x}_i$ is corrupted by randomly masking patches of data elements in $x_i$ and setting the masked region to 0 with probability $p\geq 0$.

Let $\hat{x}$ be the reconstruction of $x$ from the corrupted $\tilde{x}$.
\begin{equation}
\hat{x} = g(h(\tilde{x}))
\end{equation}
We set the loss function to measure the residual between $x$ and $\hat{x}$ under corruption as
\begin{equation}
l(x,\hat{x}) = l(x,g(h(\tilde{x}))
\end{equation}
Then the goal of the training task of the autoencoder pair is to find the optimal solution of $W$ to reconstruct the corrupted inputs $\tilde{x}$ that minimizes the squared reconstruction loss $l$.

In order to obtain an optimal autoencoder, the main concern in our design is the relationship between the stable output of the system reconstruction and the loss function, so it is of interest to study the $h(x)$ and $g(z)$ function trajectories. When an image with all zero pixels enters the system, the output of the encoder is a zero tensor, i.e., $h(0)=0$. Consequently, the output of the decoder is also a reconstruction image with all zero pixels, i.e., $g(0)=0$. Therefore, there exists a zero reference point for discussing the encoder and decoder operators. Then, we have the following proposition.

\begin{proposition}
\label{LemmaViTlinearMapping}
There exists a linear equivalent mapping with $W_h$ to the approximate encoder $h(\cdot)$.
\end{proposition}
\begin{proof}
Expanding the nonlinear vector function $h(x)$ into a Taylor series at $0$, it yields
\begin{equation}
    h(x) = h(0) + \nabla_{x} h(0) (x) + \epsilon
\end{equation}
where $\nabla_{x}h(0)$ denotes the gradient of operator $h(\cdot)$ at $x_0$ in the direction of the vector $x$, and $\epsilon$ a higher order infinitesimal residual. When ignoring the residual and let $\nabla_{x} h(0)x =  W_h x$, it yields
\begin{equation}
    h(x)\approx W_h x +h(0)
\end{equation}
As $h(0)=0$, $h(x)$ can be represented by the mapping $W_h$, and the proof is obtained.
\end{proof}

\begin{corollary}
There exists a linear equivalent mapping with $W_g$ to the approximate encoder $g(\cdot)$.
\end{corollary}

\begin{remark}
A practical representation of the linear mapping $W_h$ above is essentially an autoencoder implemented using a linear neural network. In practice, this linear representation may be relatively inferior because of the omission of higher-order feature representations and residual terms. 
In this paper, we mainly focus on the nature of the mapping and therefore use this mapping to simplify the problem, with the main aim of discussing the relationship between the sample, the loss function, and the optimal reconstruction in federated learning.
\end{remark}

For federated learning based on client $k$, this approximate linear equivalence mapping can help to simplify the model, then approximated encoder $h(\cdot)$ can be formulated as a linear mapping $W_{h_k}$. Similarly, the decoder $g(\cdot)$ can be equated to another linear mapping $W_{g_k}$, which yields

\begin{equation}
\hat{x} = g_k(h_k(\tilde{x})) \approx W_{g_k} W_{h_k} \tilde{x} = W_k \tilde{x}
\end{equation}
then, the loss function can be formulated as
\begin{equation}
\label{clientLossEq}
l(x,\hat{x}) = l(x,g_k(h_k(\tilde{x})) = \frac{1}{2n} \sum_i^n \| x_i - W_k \tilde{x_i} \|^2
\end{equation}
Then we have the theorem below.

\begin{theorem}
The approximate optimal solution of $W_k$ on client $k$ can 
be obtained by $W_k^*$, in which
\begin{equation}
    W_k^* = \bar{X}_k\tilde{X}_k^T (\tilde{X}_k\tilde{X}_k^T)^{-1}=W^*_{gk} W^*_{hk}
\end{equation}
where
\begin{equation}
\begin{array}{rl}
\tilde{X}_k & = [\tilde{X}_{k,1},\ldots,\tilde{X}_{k,j},\ldots,\tilde{X}_{k,m}]_{n\times m}\\
\bar{X}_k & = [{X}_{k},\ldots,{X}_{k}]_{n \times m}
\end{array}
\end{equation}
\end{theorem}
\begin{proof}
With $m$ augmented corruptions, the aggregated input of client $k$ can be formulated as
\begin{equation}
\tilde{X}_k = [\tilde{X}_{k,1},\ldots,\tilde{X}_{k,j},\ldots,\tilde{X}_{k,m}]_{n\times m}    
\end{equation}
which is corresponding to the aggregated original data as
$\bar{X}_k = [{X}_{k},\ldots,{X}_{k}]_{n \times m}$.
Let $\hat{X}=W_k \tilde{X}$ be the aggregated reconstruction of $X$.
Then the transformed loss function of the training with the aggregation can be formulated as
\begin{equation}
\label{clientAggregationLossEq}
l(\bar{X}_k,\hat{X}_k) =\frac{1}{2}\|\bar{X}_k-W_k\tilde{X}_k\|^2
=\frac{1}{2} tr
\left[ (\bar{X}_k-W_k\tilde{X}_k)^T(\bar{X}_k-W_k\tilde{X}_k)\right]
\end{equation}
This loss function is a convex function that can reach a minimum value when its derivative is 0. Therefore, with $\nabla_{W_k} l(\bar{X},\hat{X})=0$, it yields
\begin{equation}
\label{lfmin_k}
\begin{array}{rl}
2\nabla_{W_k} l(\bar{X}_k,\hat{X}_k) 
& = \nabla_{W_k} tr\left[ (\bar{X}_k-W_k\tilde{X}_k)^T(\bar{X}_k-W_k\tilde{X}_k)\right]\\
& = \nabla_{W_k} tr(\bar{X}^T_k \bar{X}_k - \tilde{X}^T_k W_k^T \bar{X}_k - \bar{X}^T_k W_k \tilde{X}_k + \tilde{X}^T_k W_k^T W_k \tilde{X}_k)\\
& = \nabla_{W_k} tr(W_k\tilde{X}_k\tilde{X}_k^T W^T_k) - 2 \nabla_{W_k}  tr(W_k \bar{X}_k \tilde{X}^T_k )\\
& = 2 W_k\tilde{X}_k\tilde{X}^T_k - 2 \bar{X}_k \tilde{X}^T_k=0\\
\end{array}
\end{equation}
Solving the \eqref{lfmin_k} yields $W^*_k = \bar{X}_k\tilde{X}_k^T (\tilde{X}_k\tilde{X}_k^T)^{-1}$ as the completion of the proof.
\end{proof}
\begin{theorem}
The overall approximate optimal solution of $W$ on $K$ clients can 
be obtained by $W_A^*$, in which
\begin{equation}
    W_A^* = \bar{X}\tilde{X}^T (\tilde{X}\tilde{X}^T)^{-1}
\end{equation}
where
\begin{equation}
\begin{array}{rl}
\tilde{X} & =[\tilde{X}_{1,1}, \ldots, \tilde{X}_{k,j},\ldots, \tilde{X}_{K,m}]_{K\times n \times m}\\
\bar{X} & =[X, \ldots, X]_{K\times n \times m}
\end{array}
\end{equation}
\end{theorem}

\begin{proof}
The combination of data after aggregating all variation samples can be obtained by $\tilde{X}=[\tilde{X}_{1,1}, \ldots, \\ \tilde{X}_{k,j},\ldots, \tilde{X}_{K,m}]_{K\times n \times m}$,
corresponding to the same scale of the original sample data aggregation as a whole can be expressed as
$\bar{X}=[X, \ldots, X]_{K\times n \times m}$.
With the modeling and analytic deduction similar to \eqref{clientAggregationLossEq} and \eqref{lfmin_k}, the optimal solution can be obtained.
\end{proof}

\begin{proposition}
When random sampling is used to corrupt the data X, the performance of the established optimal solution is proportional to the number of clients.
\end{proposition}
\begin{remark}
As chunking and random sampling are applied to the raw data, each client enters each training step with a single random sample of the raw data it carries, resulting in a different variant of the data. If the sampling is different each time, then the original sample can thus be augmented to produce genes to cover different variant versions of the data. The corresponding $\tilde{X}$ is used to solve the optimal reconstruction mapping and the corresponding $\bar{X}$ are then enriched with elements, resulting in mapping solutions that better cover the different inputs.
\end{remark}
\begin{proposition}
The number of variant $m$ has its analytic upper bound as
\begin{equation}
\sup m = n\times \binom{B}{b} = n \frac{B!}{b!(B-b)!}    
\end{equation}
where $B$ refers to the number of all patches in $\{x_i\}_n$ after chunking segmentation,  $b$ the number of selected patches for encoding, and $\binom{B}{b}$ is a combination with a selection of patches from $B$ distinct patches.
\end{proposition}

\begin{remark}
Although the number of versions of the variation $m$ is limited, disordering all the versions at each step of the training can still enhance the performance of the system learning as dissimilar patches can be rearranged in different orders to form different versions of permutation to augment the data for training in each epoch. 
\end{remark}

\section{Experiments}
In this section, we demonstrate our experimental results. Section \ref{section 5.1} introduces our experimental setup, including datasets, set up, and the state-of-the-art baselines. Section \ref{section 5.2} compares accuracies, model parameters and GFLOPs of different FSSL methods with our FedMAE. Moreover, we demonstrate the impact of the number of total clients and the statistical heterogeneity for both FedMAE and other FSSL methods in Section \ref{section 5.3} and \ref{section 5.4}. 
In Section \ref{section 5.5}, we show the scalability of the amount of data on the server for FedMAE. 
Additionally, we present various experimental results in \ref{section 5.6}, \ref{section 5.7} and \ref{section 5.8} on many parameters (including the number of local epochs, the participation ratio and the number of training rounds) and illustrate how they influence the training results of FedMAE. 

\subsection{Experimental Setup} 
\label{section 5.1}
\noindent\textbf{Datasets}
We conduct our experiments using several datasets. We use ImageNet \cite{deng2009imagenet} with 1000 classes and Mini-ImageNet with 100 classes, and both of their input size is 224x224. The total number of mini-ImageNet is 60000, which are chosen from ImageNet by the method of \cite{vinyals2016matching}, while the ImageNet has 1,281,167 training images and 50000 test images.
We also use other public datasets, including CIFAR10, CIFAR100 \cite{krizhevsky2009learning} and Mini-INAT2021. 
There are 60,000 32x32 color images in CIFAR10 and CIFAR100, where the training set has 50,000 images, and the test set has 10,000 images. CIFAR10 is used for the task of 10-class image classification, while CIFAR100 is utilized for the 100-class image classification task.
On the other hand, Mini-INAT2021 \cite{inat2021} is also a large-scale dataset, which is the more accessible version of INaturalist Dataset 2021. Compared to the original dataset which contains nearly 2.7M training images for 10000 natural species, the number of images in Mini-INAT2021 has been cut down to 50 training samples and 10 validation samples per specie for a total of 500K training images and 100K validation images. The maximum dimension of these images is 800 pixels. 

\noindent\textbf{Setup}
 Our experiments consist of two phases: federated pre-training and downstream fine-tuning. In the pre-training phase, we use the Mini-ImageNet dataset if our results need to be compared with the results of other baselines. Otherwise, we use the ImageNet dataset to evaluate the effect of various parameters on the performance of our method. 
The pre-training dataset is divided into K clients. When the data held by each client is assumed to be IID distributed among each other, the division of the dataset requires that each client holds images of all categories. When the data division is assumed to follow non-IID distribution, we divide the dataset by taking a sample of the Dirichet distribution's class priors \cite{hsu2019measuring}. More heterogeneous division can be made by specifying smaller Dirichet parameter $\alpha$ during sampling. 
Unless otherwise stated, the federated learning conducted in the pre-training phase adopts the following settings. The number of training rounds is set to 200, the number of local epochs is set to 10, the total number of local clients is set to 100, and the participation ratio is set to 0.05. In the downstream fine-tuning phase, we replace the decoder part of the pre-trained backbone with the head layer specified by the downstream task and retrain the weights and biases of all layers. For all downstream tasks, the model is trained for 100 epochs at a batch size of 64. In addition, for the supervised evaluation, we fine-tuned the new model using 100\% of the labeled data. For the semi-supervised evaluation, we use 10\% of the labeled data for fine-tuning. Since the CIFAR10 and CIFAR100 datasets are relatively simple, we only test supervised evaluation.

\noindent\textbf{Baselines}
To fairly evaluate the proposed FedMAE framework, we compare it with the following state-of-the-art FSSL benchmarks. 
\textbf{1) Fed-SimSiam}: a naive combination of SimSaim and FL.
\textbf{2) Fed-SimCLR}:  a naive combination of SimCLR and FL.
\textbf{3) FedU \cite{zhuang2021collaborative}}: Using the divergence-aware predictor updating with self-supervised BYOL \cite{grill2020bootstrap}. \textbf{4) FedEMA \cite{zhuang2021divergence}}:  Using EMA of a global model to adaptively update online client networks. \textbf{5) Orchestra \cite{DBLP:conf/icml/LubanaTKDM22}}:  a novel clustering-based FSSL technique. 
For downstream tasks' evaluation, we fine-tune the entire model using downstream labeled data (10\% or 100\% of the total dataset).


\begin{table}[h!]
\caption{Accuracy (\%) in IID ($\alpha=0$) and Non-IID ($\alpha=1e-1$) of FL settings on CIFAR10, CIFAR100, Mini-INAT and ImageNet datasets. We use Mini-ImageNet for pre-training, and evaluate models using the popular semi-supervised fine-tuning with 10\% labeled data and fully-supervised fine-tuning. For pre-training period, the local clients are 100, the training rounds of pre-training are 200,  the local training epochs are 10 in each training round. For the downstream task, the total downstream training epochs are 100. }
\label{tab1}
\centering
\small
\renewcommand\tabcolsep{4.5pt}
\begin{tabular}{lcccccccccccc}
\hline
\multicolumn{1}{c}{} &
  \multicolumn{2}{c}{CIFAR10} &
  \multicolumn{2}{c}{CIFAR100} &
  \multicolumn{4}{c}{Mini-INAT} &
  \multicolumn{4}{c}{ImageNet} \\ \cline{2-13} 
\multicolumn{1}{c}{} &
  IID &
  non-IID &
  IID &
  non-IID &
  \multicolumn{2}{c}{IID} &
  \multicolumn{2}{c}{non-IID} &
  \multicolumn{2}{c}{IID} &
  \multicolumn{2}{c}{non-IID} \\ \cline{2-13}
\multicolumn{1}{c}{\multirow{-3}{*}{Method}} &
  100\% &
  100\% &
  100\% &
  100\% &
  10\% &
  100\% &
  10\% &
  100\% &
  10\% &
  100\% &
  10\% &
  100\% \\ \hline
Fed-SimSiam(\%) &
  {89.91} &
  {89.58} &
  {68.52} &
  {71.46} &
  {0.13} &
  {32.57} &
  {0.10} &
  {37.43} &
  {29.84} &
  {65.26} &
  {27.81} &
  {64.87} \\
Fed-SimCLR(\%) &
  {89.54} &
  {90.39} &
  {67.20} &
  {71.24} &
  {0.15} &
  {37.70} &
  {0.13} &
  {37.60} &
  {31.84} &
  {65.47} &
  {29.50} &
  {65.32} \\
FedU(\%) &
  {77.43} &
  {72.02} &
  {40.40} &
  {38.44} &
  {0.04} &
  {37.88} &
  {0.05} &
  {37.61} &
  {33.69} &
  {65.34} &
  {31.15} &
  {65.34} \\
FedEMA(\%) &
  {70.73} &
  {71.00} &
  {40.78} &
  {41.13} &
  {0.17} &
  {38.40} &
  {016} &
  {37.43} &
  {28.65} &
  {65.24} &
  {27.89} &
  {65.35} \\
Orchestra(\%) &
  {88.87} &
  \textbf{90.66} &
  {72.11} &
  {72.27} &
  {0.12} &
  {38.74} &
  {0.15} &
  {39.23} &
  {32.24} &
  {65.02} &
  {33.15} &
  {66.50} \\ \hline
FedMAE(\%) &
  \textbf{90.92} &
  {90.57} &
  \textbf{73.33} &
  \textbf{74.11} &
  \textbf{1.53} &
  \textbf{46.01} &
  \textbf{1.54} &
  \textbf{41.27} &
  \textbf{37.93} &
  \textbf{77.60} &
  \textbf{38.72} &
  \textbf{77.79} \\ \hline
\end{tabular}
\end{table}

\begin{table}[h!]
\caption{Comparison of FedMAE's Model Parameters and GFLOPs with FSSL methods. We compute the model parameters and GFLOPs using images with the size of 224x224 as input.}
\label{tab2}
\centering
\small
\renewcommand\tabcolsep{3.0pt}
\begin{tabular}{lccccccc}
\hline
                    & Fed-SimSiam & Fed-SimCLR & FedU   & FedEMA & Orchestra &  \begin{tabular}[c]{@{}c@{}}FedMAE\\ (Pre-train)\end{tabular} &
  \begin{tabular}[c]{@{}c@{}}FedMAE\\(Downstream)\end{tabular} \\ \hline
Model Params & 12.03M      & 11.70M     & 38.47M & 38.47M & 11.84M    & \textbf{11.62M} & \textbf{39.97M} \\
GFLOPs              & 3.65        & 3.65       & 7.40   & 7.40   & 7.31      & \textbf{1.23}  & \textbf{7.39}   \\ \hline
\end{tabular}
\end{table}

\begin{figure}[h!]
  \centering
  \includegraphics[width=1.0\textwidth]{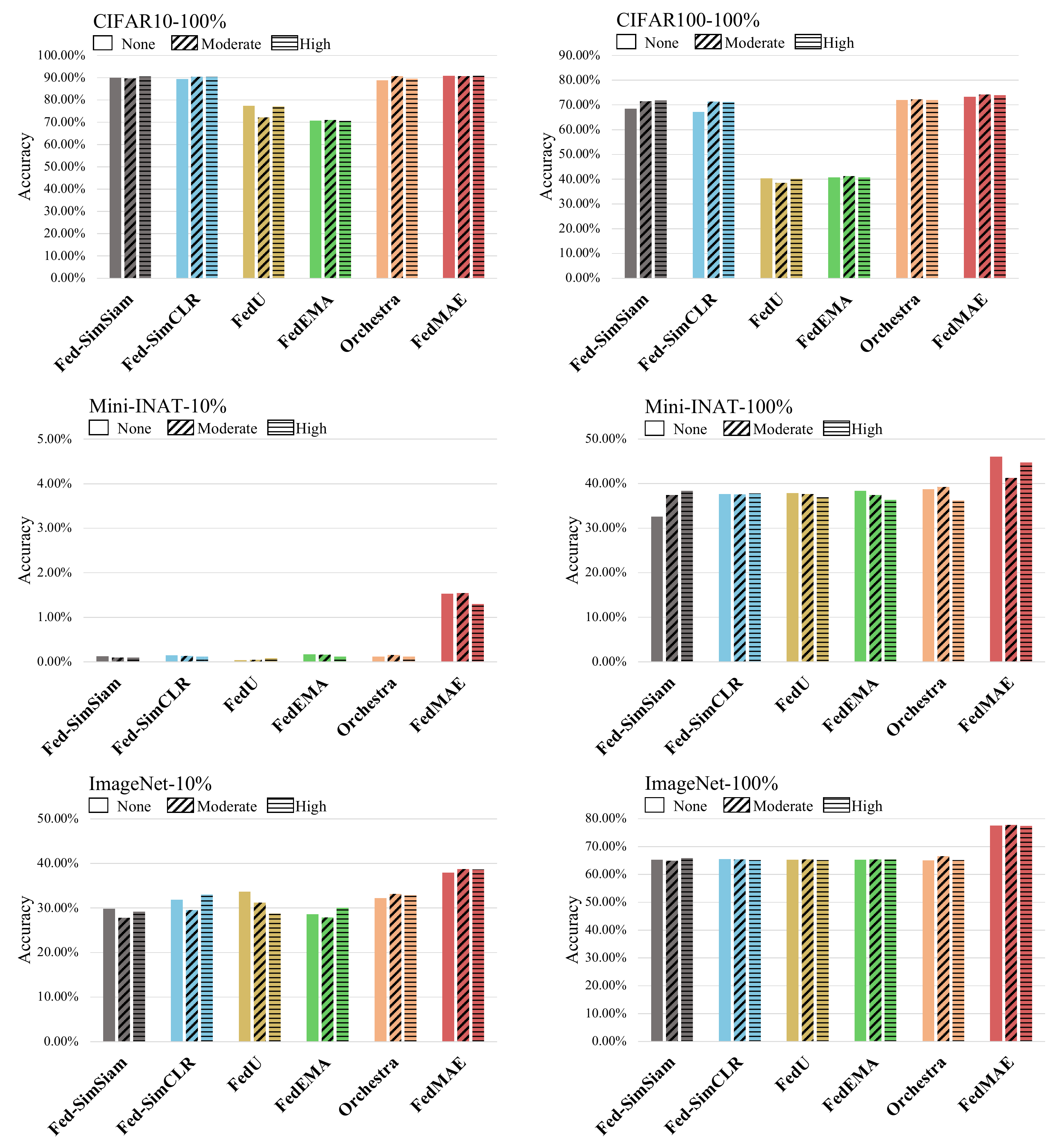}
  \caption{Sensitivity of Statistical Heterogeneity on CIFAR10, CIFAR100, Mini-INAT and ImageNet. The pre-training dataset is Mini-ImageNet. FedMAE does not lose performance with increased heterogeneity.}
  \label{fig3}
\end{figure}

\subsection{Comparison with FSSL Methods}
\label{section 5.2}
We first assess the accuracies, model parameters and GFLOPs of different FSSL methods using fine-tune protocol and transfer learning evaluation. 
In this experiment, we examine FedMAE's performance in IID and Non-IID scenarios since cross-device FL is our key purpose for developing FedMAE. The pre-training dataset used in FL is Mini-ImageNet, and the downstream tasks' datasets are CIFAR10, CIFAR100, Mini-INAT and ImageNet. 
Table \ref{tab1} shows our FedMAE can achieve state-of-the-art results on small image datasets such as CIFAR10 and CIFAR100. Meanwhile, FedMAE demonstrates exceptional performance on large image datasets, such as Mini-INAT and ImageNet, surpassing other FSSL techniques in both IID and non-IID scenarios in terms of accuracy. Specifically, in the IID environment with 100\% labeled data, FedMAE achieves an accuracy of 46.01\% on Mini-INAT, which is 7.27\% higher than the accuracy obtained by other methods. In the non-IID scenario with 100\% labeled data, FedMAE achieves an accuracy of 77.79\% on ImageNet, which is 11.29\% higher than the accuracy obtained by other methods.
Table \ref{tab2} compares the differences between FedMAE and other FSSL methods in terms of model parameters and GFLOPs. The pre-training model of FedMAE is not only the smallest in parameters but also the smallest in GFLOPs, which means faster transmission between local clients and the server, and fewer computing power requirements for local clients. Although the downstream model of FedMAE has a large number of parameters due to the cascade design, the computation resources required for its training and subsequent inference are acceptable among all models at present, which can be deployed for inference in local clients.



\subsection{Statistical Heterogeneity.}
\label{section 5.3}
We use 100 clients and analyze three levels of heterogeneity by setting $\alpha$ to 0 (none/IID), 1e-1 (moderate), and 1e-3 (high). Results are shown in Figure \ref{fig3}. The data distribution in federated learning is dispersed and most current FSSL methods are based on contrastive learning, so SSL approaches tend to be sensitive to heterogeneity and therefore adopt certain strategies for this challenge. In contrast, FedMAE is robust and its performance generally does not degrade with increasing heterogeneity. In fact, the pre-trained model of FedMAE is an image reconstruction method that uses an autoencoder, which is not sensitive to image classification, FedMAE learns how to reconstruct the image itself rather than the class features of images in the period of pre-training.

\begin{figure}[h!]
  \centering
  \includegraphics[width=1.0\textwidth]{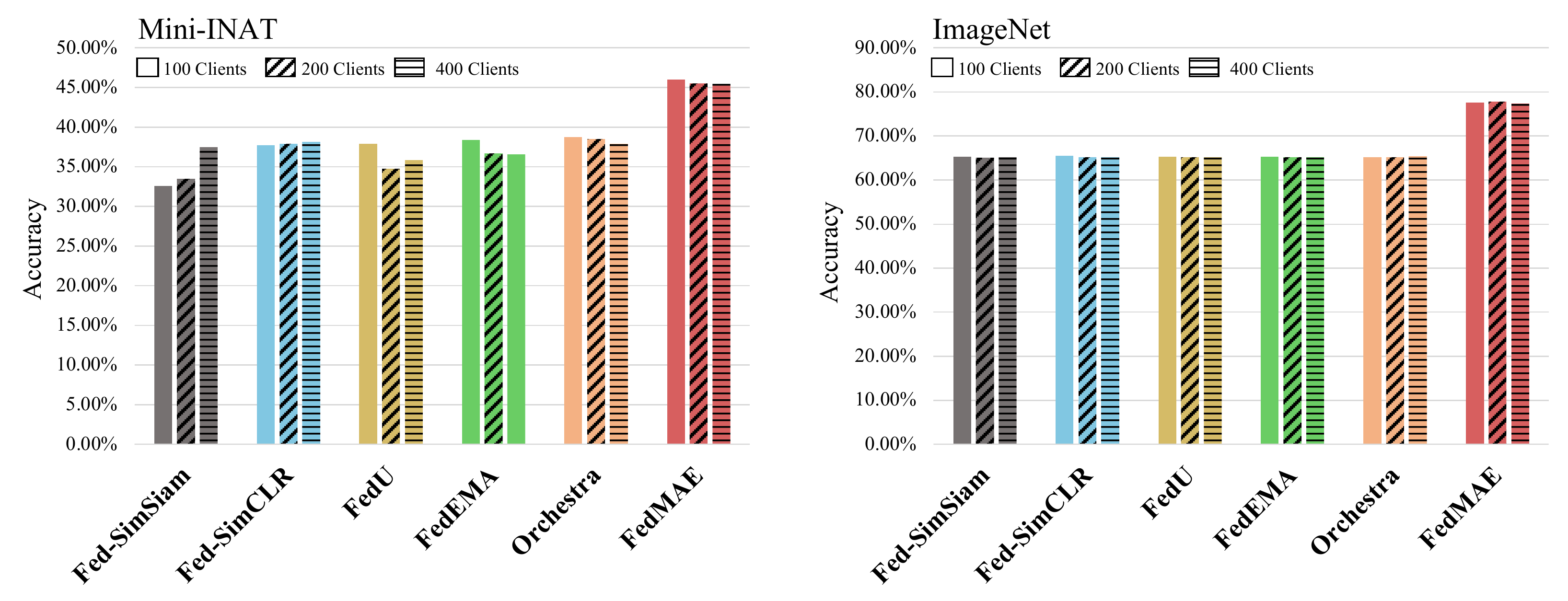}
  \caption{Scalability with the number of clients on Mini-INAT and ImageNet. The pre-training dataset is Mini-ImageNet. FedMAE outperforms other FSSL methods and is generally robust to a number of clients in the federation.}
  \label{fig4}
\end{figure}

\subsection{Number of Clients.}
\label{section 5.4}
Changing the number of clients also affects the performance of the pre-trained model in FL. We increased the number of local clients from 100 to 400 with IID data distribution, still selecting 5 clients with 10 epochs in each training round, and the total number of training rounds are 200. Results are provided in Figure \ref{fig4}. We find that, unlike other methods which significantly degrade performance as the number of clients increases, FedMAE achieves similar performance or a gentle decline, which indicates FedMAE can still learn a lot of useful knowledge on relatively dispersed and small amounts of data.

\begin{table}[h!]
\caption{Accuracy (\%) with the amount of data on the server. The pre-training dataset is Mini-ImageNet.}
\label{tab3}
\centering
\small
\renewcommand\tabcolsep{2.0pt}
\begin{tabular}{lcccccccccccccccc}
\hline
 &
  \multicolumn{4}{c}{CIFAR10} &
  \multicolumn{4}{c}{CIFAR100} &
  \multicolumn{4}{c}{Mini-INAT} &
  \multicolumn{4}{c}{ImageNet} \\ \hline
Data on Server &
   0 &
   1000 &
   5000 &
   10000 &
   0 &
   1000 &
   5000 &
   10000 &
   0 &
   1000 &
   5000 &
   10000 &
   0 &
   1000 &
   5000 &
   10000 \\
FedMAE (\%) &
    90.92 &
    91.70 &
    92.40 &
    \textbf{92.96} &
    73.93 &
    74.23 &
    75.16 &
    \textbf{76.17} &
    46.01 &
    46.33 &
    48.83 &
    \textbf{49.01} &
    77.60 &
    77.81 &
    77.84 &
    \textbf{77.90} \\ \hline
\end{tabular}
\end{table}

\subsection{Scalability with the amount of data on the server}
\label{section 5.5}
Since what one-block MAE can learn is very limited, we share a small amount of data on the server to help build a better downstream model. We tested the accuracies of downstream datasets when sharing 0, 1000, 5000 and 10000 unlabeled data on the server. In this scenario, we first use clients' data for federated training, after 200 training rounds, we cascade 5 one-block pre-trained MAE to build a five-block MAE. Afterward, sharing data on the server will be fed into the five-block MAE to train a better downstream model for downstream tasks. The training epochs here for sharing data are 100. Results are shown in Table \ref{tab3}. Although the amount of data on the server is very limited, FedMAE provides a good initialized model that can further build a better downstream model.

\begin{table}[h!]
\caption{Robustness to local epochs on CIFAR10 and CIFAR100. The pre-training dataset is ImageNet. We find that FedMAE achieves the best accuracy when local epochs are 10.}
\label{tab5}
\centering
\small
\renewcommand\tabcolsep{10.0pt}
\begin{tabular}{{lccccc}}
\hline
\begin{tabular}[c]{@{}l@{}}Number of Local Epochs\\ in each training round\end{tabular} & 1       & 5       & 10      & 15      & 20      \\ \hline
CIFAR10 (\%)              & 90.57 & 93.25 & 93.68 & 92.73 & 92.71 \\ 
CIFAR100 (\%)             & 73.41 & 76.75 & 77.07 & 75.74 & 75.78 \\ \hline
\end{tabular}
\end{table}

\subsection{Local Epochs}
\label{section 5.6}
The number of local epochs represents the number of times all data from each client is trained in each training round. According to Table \ref{tab5}, we find that increasing this federated learning attribute can lead to an improvement in fine-tuning accuracy and the best result is achieved by setting the number of local epochs to 10. However, local training may be limited to a small number of epochs because of the resource limitations on clients. If we need to take this into account, it might be more appropriate to set the number of local epochs to 5, as the improvement from increasing the number of local epochs from 5 to 10 is very small.

\begin{table}[h!]
\caption{Sensitivity to participation ratio on CIFAR10 and CIFAR100. The pre-training dataset is ImageNet. While the accuracies of clients decrease at smaller participation ratios normally, we find FedMAE is almost no degradation.}
\label{tab6}
\centering
\small
\renewcommand\tabcolsep{10.0pt}
\begin{tabular}{{lccccc}}
\hline
\begin{tabular}[c]{@{}l@{}}Number of Clients \\ in each training round\end{tabular} & 5 & 10 & 15 & 20 & 25 \\ \hline
CIFAR10 (\%)              & 91.77 & 91.57 & 91.98 & 91.56 & 91.66 \\ 
CIFAR100 (\%)             & 75.68 & 75.32 & 75.77 & 75.75 & 75.02 \\ \hline
\end{tabular}
\end{table}

\subsection{Participation Ratio}
\label{section 5.7}
The participation ratio indicates the number of clients connecting to the server in each training round. In this experiment, our method follows the aggregation strategy of other FSSL methods to average the client models trained in each round, while Table \ref{tab6} shows that simply increasing this federated learning attribute does not produce a boost in fine-tuning accuracy as in other FSSL methods which uses CNN architecture. The reason why this attribute has no effect to the results is that the pre-train model is very short which only has one block autoencoder, and the learning rate is very small. If we increase the depth of the pre-trained backbone and the participation ratio,  the fine-tuning accuracy may also be improved.

\begin{table}[h!]
\caption{Accuracy with the number of training rounds for pre-training on CIFAR10 and CIFAR100. The pre-training dataset is ImageNet. The accuracy will improve with the number of training rounds and finally reach a stable value.}
\label{tab7}
\centering
\small
\renewcommand\tabcolsep{10.0pt}
\begin{tabular}{lccccc}
\hline
\begin{tabular}[c]{@{}l@{}}Number of \\ Training Rounds\end{tabular} & 200     & 400     & 600     & 800     & 1000    \\ \hline
CIFAR10 (\%)                & 91.44 & 92.43 & 92.49 & 92.17 & 92.76 \\
CIFAR100 (\%)                 & 75.11 & 75.70 & 76.07 & 76.19 & 76.01 \\ \hline\hline
\begin{tabular}[c]{@{}l@{}}Number of \\ Training Rounds\end{tabular} & 1200    & 1400    & 1600    & 1800    & 2000    \\ \hline
CIFAR10 (\%)                  & 92.98 & 92.83 & 92.99 & 93.18 & 93.02 \\
CIFAR100 (\%)                 & 76.42 & 76.26 & 76.20 & 76.46 & 76.48 \\ \hline
\end{tabular}
\end{table}

\subsection{Number of Training Rounds}
\label{section 5.8}
According to Table \ref{tab7}, the fine-tuning accuracy for the downstream tasks increases as the number of training rounds  increases in pre-training. This increase trend is more pronounced in the first 1200 rounds, while eases off in the later rounds. Therefore, for our method FedMAE, we consider a good Nash equilibrium between training cost and fine-tuning accuracy is 1200 training rounds.

\begin{table}[h!]
\caption{Comparison on different model depth between un-pretrained model and FedMAE with pre-trained model cascading. The pre-training dataset is Mini-ImageNet.}
\label{tab4}
\centering
\small
\renewcommand\tabcolsep{4.5pt}
\begin{tabular}{llcccccccccc}
\hline
 \multicolumn{2}{l}{Model Depth}  & 1     & 2     & 3     & 4     & 5     & 6     & 7     & 8     & 9     & 10    \\ \hline
\multirow{2}{*}{CIFAR10}  & Un-pretrained & 66.49 & 71.86 & 75.97 & 78.81 & 80.19 & 81.73 & 82.29 & 82.87 & 82.83 & 83.09 \\
 & Pre-trained & \textbf{72.21} & \textbf{80.48} & \textbf{85.84} & \textbf{88.51} & \textbf{90.02} & \textbf{91.08} & \textbf{91.30} & \textbf{91.55} & \textbf{91.61} & \textbf{91.74} \\
\multirow{2}{*}{CIFAR100} & Un-pretrained & 40.41 & 49.76 & 55.49 & 59.18 & 60.76 & 62.20 & 63.17 & 63.99 & 63.89 & 63.91 \\
 & Pre-trained & \textbf{47.03} & \textbf{60.26} & \textbf{66.93} & \textbf{71.43} & \textbf{73.55} & \textbf{73.46} & \textbf{73.34} & \textbf{73.95} & \textbf{73.96} & \textbf{74.00} \\ \hline
\end{tabular}
\end{table}

\section{Ablation Study}
\label{section 6}

\subsection{Model Depth in Downstream Model}
We further conducted additional experiments on the depth of the downstream model. Since the pre-trained model is an MAE with only one block, and its learning ability to extract features is very limited, so we want to know whether the pre-trained one-block encoder has a positive effect on the deeper blocks of a multi-block encoder. Therefore, we cascade different numbers of pre-trained models to compare with models of the same depth but without pre-training. The results are shown in Table \ref{tab4}, which indicate that the pre-trained model still has a great influence on the overall downstream model as the depth increases. Although the increase in model depth would have improved the learning ability of the model, the result is obvious that the one-block pre-trained model improves the learning ability of each block of the downstream model with multi-blocks.

\begin{table}[h!]
\caption{The influence of the pre-trained model depth (i.e. the number of pre-trained Transformer blocks in downstream model) on fine-tuning accuracy. The pre-training dataset is Mini-ImageNet.}
\label{tab8}
\centering
\small
\renewcommand\tabcolsep{1.0pt}
\begin{tabular}{lccccccc}
\hline
\multirow{2}{*}{Model Depth} & \multicolumn{5}{c}{Downstream Model Blocks}                                                        & \multicolumn{2}{c}{Accuracy (\%)} \\
                             & Block 1              & Block 2              & Block 3              & Block 4       & Block 5       & CIFAR10         & CIFAR100        \\ \hline
\multirow{2}{*}{1}           & Un-pretrained        & -                    & -                    & -             & -             & 66.49           & 40.41           \\
                             & \textbf{Pre-trained} & -                    & -                    & -             & -             & \textbf{72.21}  & \textbf{47.03}  \\ \hline
\multirow{3}{*}{2}           & Un-pretrained        & Un-pretrained        & -                    & -             & -             & 71.86           & 49.76           \\
                             & \textbf{Pre-trained} & Un-pretrained        & -                    & -             & -             & 77.28           & 54.90           \\
                             & \textbf{Pre-trained} & \textbf{Pre-trained} & -                    & -             & -             & \textbf{80.48}  & \textbf{60.26}  \\ \hline
\multirow{4}{*}{3}           & Un-pretrained        & Un-pretrained        & Un-pretrained        & -             & -             & 75.97           & 55.49           \\
                             & \textbf{Pre-trained} & Un-pretrained        & Un-pretrained        & -             & -             & 79.97           & 61.04           \\
                             & \textbf{Pre-trained} & \textbf{Pre-trained} & Un-pretrained        & -             & -             & 83.45           & 65.26           \\
                             & \textbf{Pre-trained} & \textbf{Pre-trained} & \textbf{Pre-trained} & -             & -             & \textbf{85.84}  & \textbf{66.93}  \\ \hline
\multirow{5}{*}{4}           & Un-pretrained        & Un-pretrained        & Un-pretrained        & Un-pretrained & -             & 78.81           & 59.18           \\
                             & \textbf{Pre-trained} & Un-pretrained        & Un-pretrained        & Un-pretrained & -             & 82.79           & 64.08           \\
                             & \textbf{Pre-trained} & \textbf{Pre-trained} & Un-pretrained        & Un-pretrained & -             & 85.09           & 67.38           \\
                             & \textbf{Pre-trained} & \textbf{Pre-trained} & \textbf{Pre-trained} & Un-pretrained & -             & 87.67           & 70.19           \\
                   & \textbf{Pre-trained} & \textbf{Pre-trained} & \textbf{Pre-trained} & \textbf{Pre-trained} & -                    & \textbf{88.51} & \textbf{71.43} \\ \hline
\multirow{6}{*}{5} & Un-pretrained        & Un-pretrained        & Un-pretrained        & Un-pretrained        & Un-pretrained        & 80.19          & 60.76          \\
                             & \textbf{Pre-trained} & Un-pretrained        & Un-pretrained        & Un-pretrained & Un-pretrained & 84.61           & 66.04           \\
                             & \textbf{Pre-trained} & \textbf{Pre-trained} & Un-pretrained        & Un-pretrained & Un-pretrained & 85.75           & 68.70           \\
                   & \textbf{Pre-trained} & \textbf{Pre-trained} & \textbf{Pre-trained} & Un-pretrained        & Un-pretrained        & 87.29          & 71.31          \\
                   & \textbf{Pre-trained} & \textbf{Pre-trained} & \textbf{Pre-trained} & \textbf{Pre-trained} & Un-pretrained        & 89.25          & 73.00          \\
                   & \textbf{Pre-trained} & \textbf{Pre-trained} & \textbf{Pre-trained} & \textbf{Pre-trained} & \textbf{Pre-trained} & \textbf{90.02} & \textbf{73.55} \\ \hline
\end{tabular}
\end{table}

\subsection{Depth of Pre-Trained Backbone}
\label{section A.6}
The results of the previous ablation study show that the fine-tuning accuracy is influenced by the depth of the downstream model, and our method FedMAE still has an effect on the downstream model as the depth increases. Based on this finding, we continue to test whether changing the number of pre-trained blocks in the downstream model could also affect the fine-tuning accuracy for downstream tasks. The model depth in this session is from one to five. The experimental results in Table \ref{tab8} demonstrate that every pre-trained Transformer block in the cascaded encoder contributes to the improvement of fine-tuning accuracy. As the depth of the model increases, the more pre-trained blocks are cascaded in downstream models, the higher fine-tuning accuracy the downstream model has.

\section{Conclusion}

We introduce a novel framework FedMAE, addressing the training of large-scale distributed unlabeled datasets in federated learning. Our pre-training model only contains a one-block encoder and a one-block decoder and  uses a masking strategy for training data, so it is easy to deploy in various lightweight clients. Pre-training models in FL can be trained on different clients asynchronously without average aggregation. Pre-trained models are cascaded in the server to build a multi-block downstream model. Theoretical analysis and experimental results on image reconstruction and classification show that our FedMAE achieves a considerable boost compared to previous FSSL methods.

\printbibliography

@inproceedings{he2020momentum,
  title={Momentum contrast for unsupervised visual representation learning},
  author={He, Kaiming and Fan, Haoqi and Wu, Yuxin and Xie, Saining and Girshick, Ross},
  booktitle={Proceedings of the IEEE/CVF conference on computer vision and pattern recognition},
  pages={9729--9738},
  year={2020}
}

@inproceedings{he2022masked,
  title={Masked autoencoders are scalable vision learners},
  author={He, Kaiming and Chen, Xinlei and Xie, Saining and Li, Yanghao and Doll{\'a}r, Piotr and Girshick, Ross},
  booktitle={Proceedings of the IEEE/CVF Conference on Computer Vision and Pattern Recognition},
  pages={16000--16009},
  year={2022}
}

@inproceedings{chen2020simple,
  title={A simple framework for contrastive learning of visual representations},
  author={Chen, Ting and Kornblith, Simon and Norouzi, Mohammad and Hinton, Geoffrey},
  booktitle={International conference on machine learning},
  pages={1597--1607},
  year={2020},
  organization={PMLR}
}

@inproceedings{chen2021exploring,
  title={Exploring simple siamese representation learning},
  author={Chen, Xinlei and He, Kaiming},
  booktitle={Proceedings of the IEEE/CVF Conference on Computer Vision and Pattern Recognition},
  pages={15750--15758},
  year={2021}
}

@article{grill2020bootstrap,
  title={Bootstrap your own latent-a new approach to self-supervised learning},
  author={Grill, Jean-Bastien and Strub, Florian and Altch{\'e}, Florent and Tallec, Corentin and Richemond, Pierre and Buchatskaya, Elena and Doersch, Carl and Avila Pires, Bernardo and Guo, Zhaohan and Gheshlaghi Azar, Mohammad and others},
  journal={Advances in neural information processing systems},
  volume={33},
  pages={21271--21284},
  year={2020}
}

@inproceedings{zhuang2021collaborative,
  title={Collaborative unsupervised visual representation learning from decentralized data},
  author={Zhuang, Weiming and Gan, Xin and Wen, Yonggang and Zhang, Shuai and Yi, Shuai},
  booktitle={Proceedings of the IEEE/CVF International Conference on Computer Vision},
  pages={4912--4921},
  year={2021}
}

@inproceedings{zhuang2021divergence,
  title={Divergence-aware Federated Self-Supervised Learning},
  author={Zhuang, Weiming and Wen, Yonggang and Zhang, Shuai},
  booktitle={International Conference on Learning Representations},
  year={2021}
}

@inproceedings{DBLP:conf/icml/LubanaTKDM22,
  author    = {Ekdeep Singh Lubana and
               Chi Ian Tang and
               Fahim Kawsar and
               Robert P. Dick and
               Akhil Mathur},
  editor    = {Kamalika Chaudhuri and
               Stefanie Jegelka and
               Le Song and
               Csaba Szepesv{\'{a}}ri and
               Gang Niu and
               Sivan Sabato},
  title     = {Orchestra: Unsupervised Federated Learning via Globally Consistent
               Clustering},
  booktitle = {International Conference on Machine Learning, {ICML} 2022, 17-23 July
               2022, Baltimore, Maryland, {USA}},
  series    = {Proceedings of Machine Learning Research},
  volume    = {162},
  pages     = {14461--14484},
  publisher = {{PMLR}},
  year      = {2022},
  url       = {https://proceedings.mlr.press/v162/lubana22a.html},
  bibsource = {dblp computer science bibliography, https://dblp.org}
}

@article{li2020federated,
  title={Federated optimization in heterogeneous networks},
  author={Li, Tian and Sahu, Anit Kumar and Zaheer, Manzil and Sanjabi, Maziar and Talwalkar, Ameet and Smith, Virginia},
  journal={Proceedings of Machine Learning and Systems},
  volume={2},
  pages={429--450},
  year={2020}
}

@inproceedings{mcmahan2017communication,
  title={Communication-efficient learning of deep networks from decentralized data},
  author={McMahan, Brendan and Moore, Eider and Ramage, Daniel and Hampson, Seth and y Arcas, Blaise Aguera},
  booktitle={Artificial intelligence and statistics},
  pages={1273--1282},
  year={2017},
  organization={PMLR}
}

@article{dosovitskiy2020image,
  title={An image is worth 16x16 words: Transformers for image recognition at scale},
  author={Dosovitskiy, Alexey and Beyer, Lucas and Kolesnikov, Alexander and Weissenborn, Dirk and Zhai, Xiaohua and Unterthiner, Thomas and Dehghani, Mostafa and Minderer, Matthias and Heigold, Georg and Gelly, Sylvain and others},
  journal={arXiv preprint arXiv:2010.11929},
  year={2020}
}

@inproceedings{bao2021beit,
  title={BEiT: BERT Pre-Training of Image Transformers},
  author={Bao, Hangbo and Dong, Li and Piao, Songhao and Wei, Furu},
  booktitle={International Conference on Learning Representations},
  year={2021}
}

@inproceedings{doersch2015unsupervised,
  title={Unsupervised visual representation learning by context prediction},
  author={Doersch, Carl and Gupta, Abhinav and Efros, Alexei A},
  booktitle={Proceedings of the IEEE international conference on computer vision},
  pages={1422--1430},
  year={2015}
}

@inproceedings{noroozi2016unsupervised,
  title={Unsupervised learning of visual representations by solving jigsaw puzzles},
  author={Noroozi, Mehdi and Favaro, Paolo},
  booktitle={European conference on computer vision},
  pages={69--84},
  year={2016},
  organization={Springer}
}

@inproceedings{zhang2016colorful,
  title={Colorful image colorization},
  author={Zhang, Richard and Isola, Phillip and Efros, Alexei A},
  booktitle={European conference on computer vision},
  pages={649--666},
  year={2016},
  organization={Springer}
}

@inproceedings{gidaris2018unsupervised,
  title={Unsupervised Representation Learning by Predicting Image Rotations},
  author={Gidaris, Spyros and Singh, Praveer and Komodakis, Nikos},
  booktitle={International Conference on Learning Representations},
  year={2018}
}

@inproceedings{wu2018unsupervised,
  title={Unsupervised feature learning via non-parametric instance discrimination},
  author={Wu, Zhirong and Xiong, Yuanjun and Yu, Stella X and Lin, Dahua},
  booktitle={Proceedings of the IEEE conference on computer vision and pattern recognition},
  pages={3733--3742},
  year={2018}
}

@article{chen2019communication,
  title={Communication-efficient federated deep learning with layerwise asynchronous model update and temporally weighted aggregation},
  author={Chen, Yang and Sun, Xiaoyan and Jin, Yaochu},
  journal={IEEE transactions on neural networks and learning systems},
  volume={31},
  number={10},
  pages={4229--4238},
  year={2019},
  publisher={IEEE}
}

@article{zhao2018federated,
  title={Federated learning with non-iid data},
  author={Zhao, Yue and Li, Meng and Lai, Liangzhen and Suda, Naveen and Civin, Damon and Chandra, Vikas},
  journal={arXiv preprint arXiv:1806.00582},
  year={2018}
}

@article{kairouz2021advances,
  title={Advances and open problems in federated learning},
  author={Kairouz, Peter and McMahan, H Brendan and Avent, Brendan and Bellet, Aur{\'e}lien and Bennis, Mehdi and Bhagoji, Arjun Nitin and Bonawitz, Kallista and Charles, Zachary and Cormode, Graham and Cummings, Rachel and others},
  journal={Foundations and Trends{\textregistered} in Machine Learning},
  volume={14},
  number={1--2},
  pages={1--210},
  year={2021},
  publisher={Now Publishers, Inc.}
}

@article{vinyals2016matching,
  title={Matching networks for one shot learning},
  author={Vinyals, Oriol and Blundell, Charles and Lillicrap, Timothy and Wierstra, Daan and others},
  journal={Advances in neural information processing systems},
  volume={29},
  year={2016}
}

@article{hsu2019measuring,
  title={Measuring the effects of non-identical data distribution for federated visual classification},
  author={Hsu, Tzu-Ming Harry and Qi, Hang and Brown, Matthew},
  journal={arXiv preprint arXiv:1909.06335},
  year={2019}
}

@article{krizhevsky2009learning,
  title={Learning multiple layers of features from tiny images},
  author={Krizhevsky, Alex and Hinton, Geoffrey and others},
  year={2009},
  publisher={Citeseer}
}

@inproceedings{deng2009imagenet,
  title={Imagenet: A large-scale hierarchical image database},
  author={Deng, Jia and Dong, Wei and Socher, Richard and Li, Li-Jia and Li, Kai and Fei-Fei, Li},
  booktitle={2009 IEEE conference on computer vision and pattern recognition},
  pages={248--255},
  year={2009},
  organization={Ieee}
}

@inproceedings{wei2022masked,
  title={Masked feature prediction for self-supervised visual pre-training},
  author={Wei, Chen and Fan, Haoqi and Xie, Saining and Wu, Chao-Yuan and Yuille, Alan and Feichtenhofer, Christoph},
  booktitle={Proceedings of the IEEE/CVF Conference on Computer Vision and Pattern Recognition},
  pages={14668--14678},
  year={2022}
}

@inproceedings{liang2022rscfed,
  title={RSCFed: Random Sampling Consensus Federated Semi-supervised Learning},
  author={Liang, Xiaoxiao and Lin, Yiqun and Fu, Huazhu and Zhu, Lei and Li, Xiaomeng},
  booktitle={Proceedings of the IEEE/CVF Conference on Computer Vision and Pattern Recognition},
  pages={10154--10163},
  year={2022}
}

@inproceedings{jeong2020federated,
  title={Federated Semi-Supervised Learning with Inter-Client Consistency \& Disjoint Learning},
  author={Jeong, Wonyong and Yoon, Jaehong and Yang, Eunho and Hwang, Sung Ju},
  booktitle={International Conference on Learning Representations},
  year={2020}
}

@article{lan2019albert,
  title={Albert: A lite bert for self-supervised learning of language representations},
  author={Lan, Zhenzhong and Chen, Mingda and Goodman, Sebastian and Gimpel, Kevin and Sharma, Piyush and Soricut, Radu},
  journal={arXiv preprint arXiv:1909.11942},
  year={2019}
}

@article{devlin2018bert,
  title={Bert: Pre-training of deep bidirectional transformers for language understanding},
  author={Devlin, Jacob and Chang, Ming-Wei and Lee, Kenton and Toutanova, Kristina},
  journal={arXiv preprint arXiv:1810.04805},
  year={2018}
}

@misc{inat2021,
  title = {{iNaturalist} 2021 competition dataset},
  howpublished = {\url{https://github.com/visipedia/inat_comp/tree/master/2021}},
  year = {2021}
}
\end{document}